\documentclass{llncs}

\usepackage{eccv}

\usepackage{graphicx}
\usepackage{booktabs}

\usepackage[breaklinks,colorlinks,citecolor=eccvblue]{hyperref}
\begin{document}
\pagestyle{empty}

\title{MetaSR: Content-Adaptive Metadata Orchestration for Generative Super-Resolution}
\titlerunning{MetaSR: Content-Adaptive Metadata Orchestration for Generative Super-Resolution}

\author{Jiaqi Guo\inst{1} \and Mingzhen Li\inst{1} \and Haohong Wang\inst{2} \and Aggelos K. Katsaggelos\inst{1}}
\authorrunning{J.~Guo et al.}

\institute{Department of Electrical and Computer Engineering, Northwestern University, Evanston, Illinois, USA\\
\email{\{jiaqi.guo, mingzhen.li, a-katsaggelos\}@northwestern.edu}
\and
TCL\\
\email{haohong.wang@tcl.com}}
\maketitle

\begin{abstract}
We study generative super-resolution (SR) in real-world scenarios where content and degradations vary across domains, genres, and segments. For example, images and videos may alternate between text overlays, fast motion, smooth cartoons, and low-light faces, each benefiting from different forms of side information. Existing metadata-guided SR methods typically use a fixed conditioning design, which is suboptimal when useful cues are content dependent and transmission budgets are limited. We propose MetaSR, a Diffusion Transformer (DiT)-based framework that selects and injects task-relevant metadata to guide SR under resource constraints. Specifically, we use the DiT's own VAE and transformer backbone to fuse heterogeneous metadata, and adopt an efficient distillation strategy that enables one-step diffusion inference. Experiments across diverse content buckets and degradation regimes show that MetaSR outperforms reference solutions by up to 1.0~dB PSNR while achieving up to 50\% transmission bitrate saving at matched quality. We assess these gains under a rate--distortion optimization (RDO) framework that jointly accounts for sender-side bitrate and receiver/display quality metrics (e.g., PSNR and SSIM).

\end{abstract}

\section{Introduction}
\label{sec:introduction}
\begin{figure}[ht]
  \centering
  \includegraphics[width=0.90\linewidth]{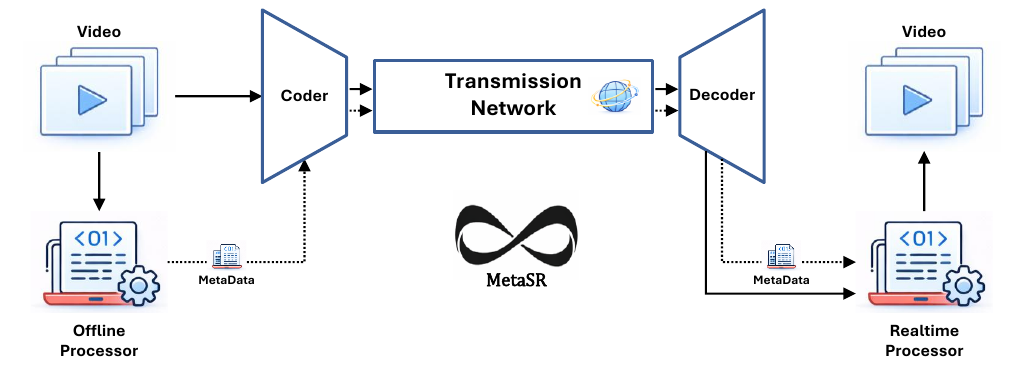}
  \caption{Overview of the proposed sender--receiver collaborative pipeline. The sender performs content analysis and metadata generation/compression, while the receiver combines decoded metadata with transmitted low-resolution content for adaptive generative super-resolution.}
  \label{fig:sender_receiver}
\end{figure}
Super-resolution (SR) aims to recover a high-resolution (HR) image or video from a low-resolution (LR) observation. Most existing approaches treat SR as a receiver-side, pixel-centric inference problem, implicitly assuming that all recoverable information is contained within the degraded signal. However, SR is fundamentally an information reconstruction problem under incomplete observations, where missing high-frequency content cannot be deterministically inferred from pixels alone. In contrast to conventional SR systems that rely exclusively on pixel-domain inference or a single neural reconstruction model, we consider a practical networked deployment scenario and disclose a sender–receiver collaborative framework, as depicted in Fig.~\ref{fig:sender_receiver}. In this framework, a transmitting system and a receiving system operate over a bandwidth-constrained communication channel to enable high-fidelity super-resolution under realistic transmission conditions. The sender side performs computationally intensive metadata generation, representation extraction, and compression, producing compact auxiliary information that complements the transmitted low-resolution content. The receiver then decodes the metadata and combines it with the low-resolution input for super-resolution. This system-level perspective is aligned with the Rich Detail Range (RDR) concept for AI-era perceptual quality and display-oriented detail recovery~\cite{wangrich}.

A growing body of work augments SR with metadata such as Canny edges, semantic masks, text descriptions, ROI maps, references, and motion hints to guide reconstruction. However, real-world SR remains challenging because the LR input is often produced by complex and time-varying degradations, and the optimal cues for recovering fine details depend strongly on the content and segment characteristics. For example, text-heavy regions require sharp edge and character-structure guidance, high-motion sports benefit from temporal motion cues for stability, while natural scenes may rely more on semantic priors to avoid unnatural textures. These observations motivate a question that is largely underexplored in SR: given a piece of content, which metadata should be used—and at what cost—to maximize SR quality and stability?

We propose MetaSR, a content-adaptive metadata orchestration framework that explicitly addresses this question. It is built upon CogVideoX-2B~\cite{yang2024cogvideox} as a generative backbone, enabling a natural extension toward video super-resolution in future work. Moreover, following~\cite{chen2025dove}, we adopt a simple yet effective training strategy to learn a one-step LR to HR mapping. As a proof-of-concept and benchmark study, we focus on two representative metadata forms, Canny edge maps and depth maps, and thoroughly evaluate their contributions to image super-resolution (ISR). 

Our main contributions are as follows:
\begin{itemize}
  \item We formulate generative SR as content-adaptive metadata orchestration in a sender--receiver system and introduce an explicit RDO perspective grounded in operational rate--distortion theory~\cite{schuster2013rate}.
  \item We propose a unified conditioning interface that connects metadata orchestration with the generative backbone. Importantly, this interface relies solely on the native DiT architecture without structural modification, while allowing flexible integration of multiple metadata inputs.
  \item We validate the end-to-end pipeline under an RDO evaluation that jointly measures sender transmission bitrate and receiver/display quality, achieving up to 1.0~dB gain and up to 50\% bitrate saving over DOVE in challenging conditions.
\end{itemize}






\section{Related Work}
\label{sec:related_work}

\subsection{Pixel-Domain Super-Resolution}

\paragraph{Deterministic regression SR from LR pixels only.}
A dominant SR formulation treats super-resolution as a receiver-side regression problem: learn a mapping from a low-resolution (LR) observation to a high-resolution (HR) output using only pixel-domain inputs at inference time.
Early deep SR established end-to-end LR$\rightarrow$HR learning with convolutional networks, exemplified by SRCNN~\cite{dong2014srcnn}.
Subsequent work improved fidelity via deeper and more stable residual designs, e.g., VDSR ~\cite{kim2016vdsr}, and by optimizing residual architectures for SR, e.g., EDSR ~\cite{lim2017edsr}.
Channel-attention and very-deep residual composition further strengthened pixel-fidelity baselines, e.g., RCAN ~\cite{zhang2018rcan}.
Transformer-based restorers later improved long-range modeling for SR and related restoration tasks, as in SwinIR ~\cite{liang2021swinir}.
Despite architectural diversity, these pixel-only SR methods share a key assumption aligned with the draft’s motivation: all recoverable information must be inferred at the receiver from the degraded LR signal alone, with no explicit side information transmitted by a sender.

\paragraph{Perceptual SR and hallucination-prone priors without side information.}
To improve perceptual realism beyond distortion metrics, a large family of SR methods adopted adversarial and perceptual objectives. A canonical baseline is SRGAN ~\cite{ledig2017srgan}, followed by ESRGAN ~\cite{wang2018esrgan}.
Real-ESRGAN ~\cite{wang2021realesrgan} emphasizes robustness under unknown, compound degradations.
These approaches sit squarely within the perception--distortion trade-off formalized by Blau and Michaeli ~\cite{blau2018perceptiondistortion}: under severe information loss, perceptually plausible details may diverge from the true unknown HR, especially when conditioning is limited to the LR pixels.
This observation directly motivates MetaSR’s sender--receiver framing: transmitting structured metadata can act as an additional constraint to reduce ambiguity, rather than relying solely on a receiver-side prior.

\paragraph{Diffusion-based generative SR under pixel-only conditioning.}
Diffusion models reframe SR as conditional sampling, enabling multiple plausible HR reconstructions from a single LR input and leveraging strong learned priors.
Foundational diffusion formulations include DDPM ~\cite{ho2020ddpm} and latent diffusion with autoencoder compression ~\cite{rombach2022ldm}.
A representative diffusion SR method is SR3 ~\cite{saharia2021sr3}, which performs SR through iterative denoising refinement.
DDRM ~\cite{kawar2022ddrm} leverages a pretrained diffusion prior for inverse problems including super-resolution, while StableSR ~\cite{wang2023stablesr} adapts large pretrained diffusion priors to blind SR.
While diffusion SR improves perceptual plausibility, the draft highlights a key limitation: when LR observations are heavily corrupted and pixel-only conditioning remains ambiguous, the model may generate plausible but incorrect details.
MetaSR’s conceptual contribution is orthogonal: it reduces this ambiguity by explicitly transmitting side information (metadata) under bitrate constraints and fusing it through a DiT-style backbone, rather than relying on LR pixels alone.

\subsection{Metadata-Guided Super-Resolution}

\paragraph{Edge and structure guidance for sharpening boundaries.}
A direct way to reduce SR ambiguity is to supply structural cues that constrain boundaries and high-frequency transitions.
Classic edge detection yields compact structure maps, with Canny being a canonical choice~\cite{canny1986}.
Edge-guided SR methods inject such cues into SR pipelines, e.g., DEGREE ~\cite{yang2017degree} and SREdgeNet ~\cite{kim2018sredgenet}.
These methods validate the usefulness of edges as auxiliary guidance; however, they typically assume the edge signal is always available (often extracted at the receiver) and do not model an explicit sender that selects and transmits edges under a bitrate budget.
MetaSR differs by treating edges (and other metadata) as \emph{transmittable} side information whose cost matters, motivating bitrate-aware selection and compression of metadata streams.

\paragraph{Semantic and geometric priors as metadata for class-consistent textures.}
Beyond edges, semantic or geometric metadata can guide SR toward class- or structure-consistent textures.
SFTGAN conditions SR on semantic segmentation probability maps via spatial feature modulation~\cite{wang2018sftgan}.
While effective, such pipelines commonly rely on a \emph{fixed} metadata choice (e.g., always segment-then-modulate), and typically do not consider that (i) the best metadata may vary across content segments, and (ii) metadata may be expensive to communicate in sender--receiver deployments.
MetaSR’s framing makes these constraints explicit: metadata is a controllable resource (bitrate) and the orchestrator should decide \emph{which} metadata types to include per content.

\paragraph{Reference-based SR as implicit side information.}
Reference-based SR (RefSR) uses external exemplars to provide texture information beyond the LR input.
SRNTT formulates RefSR as neural texture transfer that degrades gracefully when references are less relevant~\cite{zhang2019srntt}.
TTSR uses transformer attention between LR and reference features to transfer textures~\cite{yang2020ttsr}.
These approaches are conceptually aligned with MetaSR’s motivation (additional information can resolve SR ambiguity), but they generally assume reference availability at inference (via retrieval or provisioning) rather than modeling a sender that transmits compact, structured, verifiable metadata under explicit bitrate constraints.

\paragraph{Multi-condition diffusion control versus DiT-native conditioning.}
In diffusion ecosystems, spatial control signals (edges, depth, segmentation) have been integrated into large generative priors through auxiliary control networks.
ControlNet adds conditioning controls to pretrained diffusion models via a parallel trainable branch and zero-initialized convolutions~\cite{zhang2023controlnet}.
T2I-Adapter learns lightweight adapters for external control signals while freezing the base model~\cite{mou2023t2iadapter}.
These works demonstrate that heterogeneous metadata (edges, depth) can effectively steer diffusion generation, but they commonly introduce additional modules or conditioning pathways.
By contrast, MetaSR’s draft claims a \emph{unified conditioning interface} that relies on the \emph{native DiT architecture} (and its VAE/tokenization pathway) to fuse heterogeneous metadata without structural redesign per metadata type---a design particularly aligned with diffusion transformers such as DiT ~\cite{peebles2023dit} and video diffusion transformers such as CogVideoX ~\cite{yang2024cogvideox}.

\paragraph{Sender--receiver side information and rate constraints.}
MetaSR’s distinctive framing is not merely “conditioning helps SR,” but rather “conditioning must be selected and transmitted under a rate budget.”
This explicitly connects to classic source coding with decoder side information, notably Wyner--Ziv coding ~\cite{wyner1976wynerziv}, and to layered delivery in scalable coding, e.g., the SVC extension of H.264/AVC ~\cite{schwarz2007svc}.
Recent learned communication systems also study end-to-end sender/receiver learning under channel constraints, such as deep JSCC ~\cite{bourtsoulatze2019deepjscc}.
MetaSR differs in the \emph{form} of side information: rather than transmitting generic enhancement residuals or opaque learned latents, it emphasizes structured metadata (e.g., bi-level edges) that can be compressed with standardized tools ~\cite{itut2000t88,jpeg2001jbig2} and can be validated via a verification gate, enabling safe fallback to pixel-only conditioning when metadata is unreliable.

\subsection{Efficient and One-Step Diffusion Models}

\paragraph{Training-free acceleration of multi-step diffusion sampling.}
Standard diffusion sampling is slow due to many sequential denoising steps.
A large class of methods accelerates inference without retraining by changing the sampler.
DDIM introduces non-Markovian sampling trajectories that reduce steps~\cite{song2021ddim}.
DPM-Solver and UniPC propose higher-order solvers/predictor--corrector frameworks for fast sampling~\cite{lu2022dpmsolver,zhao2023unipc}.
These techniques are complementary to MetaSR: they speed up diffusion but generally remain multi-step and do not address the draft’s sender-side metadata transmission problem.

\paragraph{Distillation and consistency training for few-step / one-step generation.}
Another direction trains models to intrinsically generate in very few steps.
Progressive Distillation iteratively distills multi-step teachers into fewer-step students~\cite{salimans2022progressive}.
Consistency Models learn mappings that enable one-step generation while allowing multi-step refinement~\cite{song2023consistency}.
Latent Consistency Models extend consistency-style distillation to latent diffusion backbones~\cite{luo2023lcm}, and Adversarial Diffusion Distillation combines score distillation with adversarial objectives to improve quality at 1--4 steps~\cite{sauer2023add}.
In MetaSR’s draft framing, such distillation methods are enabling techniques for deployment: they make it feasible to run a powerful diffusion backbone at receiver-side latency budgets, allowing the research focus to shift toward metadata orchestration under bitrate constraints.

\paragraph{Efficient diffusion specialized for restoration and SR, including one-step diffusion SR.}
Several methods tailor diffusion acceleration specifically to restoration/SR.
ResShift redesigns the diffusion path to reduce the required number of denoising steps for SR~\cite{yue2023resshift}.
Closest to MetaSR’s deployment objective is one-step SR in video diffusion backbones: DOVE fine-tunes a pretrained video diffusion transformer (CogVideoX) and introduces a latent--pixel training strategy for one-step VSR~\cite{chen2025dove,yang2024cogvideox}.
MetaSR adopts a similar one-step philosophy (as cited in the draft) but focuses its novelty on \emph{content-adaptive, bitrate-constrained metadata orchestration} and a \emph{unified DiT-based conditioning interface} that can fuse heterogeneous transmitted metadata into the diffusion transformer backbone.

\section{Method}
\begin{figure}[t]
  \centering
  \includegraphics[width=\linewidth]{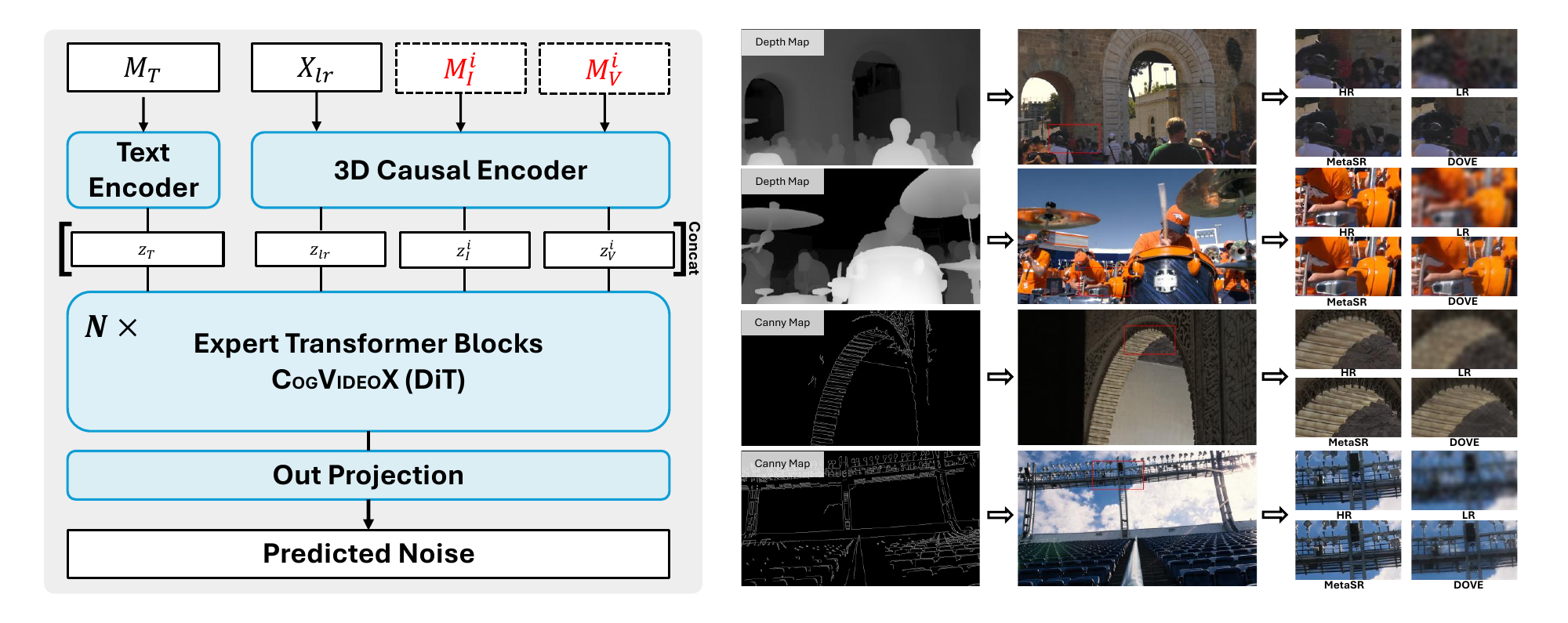}
  \caption{Overview of MetaSR. Left: schematic of how metadata information is projected into the denoising process through native DiT modules. Right: representative qualitative image SR examples under Canny/depth guidance.}
  \label{fig:framework_exp}
\end{figure}


\subsection{Backbone and one-step generative SR}
MetaSR is built upon the pretrained DiT video model CogVideoX-2B, following the same two-stage fine-tuning pipeline as DOVE~\cite{chen2025dove}.
CogVideoX uses a 3D causal VAE with encoder $\mathcal{E}$ and decoder $\mathcal{D}$ to map inputs to a latent space, and a Transformer denoiser $v_{\theta}$ for conditional generation.
Given a low-quality input $x_{\mathrm{lr}}$ (e.g., compressed or downsampled), we resize it to the target resolution and encode:
\begin{equation}
z_{\mathrm{lr}} = \mathcal{E}(x_{\mathrm{lr}}).
\end{equation}
We treat $z_{\mathrm{lr}}$ as a noisy latent at timestep $t$ and perform a single denoising step:
\begin{equation}
z_{\mathrm{sr}}
= \sqrt{\bar{\alpha}_t}\,z_{\mathrm{lr}}
- \sqrt{1-\bar{\alpha}_t}\;v_{\theta}(z_{\mathrm{lr}}, c, t),
\label{eq:onestep_vpred}
\end{equation}
where $c$ denotes conditions (i.e., text or metadata).
Finally, the output image is decoded: $x_{\mathrm{sr}} = \mathcal{D}(z_{\mathrm{sr}})$.
This preserves the original CogVideoX architecture (VAE + DiT blocks) and requires no new networks at inference.

\subsection{Two-stage latent--pixel training}
As in DOVE~\cite{chen2025dove}, we use a two-stage training strategy: (i) latent-space adaptation and (ii) pixel-space refinement.
In Stage 1, given $(x_{\mathrm{lr}},x_{\mathrm{hr}})$ we encode the HR target into latents $z_{\mathrm{hr}}=\mathcal{E}(x_{\mathrm{hr}})$ and minimize
\begin{equation}
\mathcal{L}_{\mathrm{s1}}
= \frac{1}{\|z_{\mathrm{hr}}\|_2^2}\,\big\|z_{\mathrm{sr}}-z_{\mathrm{hr}}\big\|_2^2,
\label{eq:stage1_latent}
\end{equation}
where $z_{\mathrm{sr}}$ is from Eq.~\eqref{eq:onestep_vpred}.
The VAE is frozen and only the DiT denoiser is updated.
In Stage 2, we refine in pixel space to recover fine details:
for images we use 
\begin{equation}
\mathcal{L}_{\mathrm{s2,img}}
= \|\hat{x}_{\mathrm{sr}} - \hat{x}_{\mathrm{hr}}\|_2^2
+ \lambda_{1}\,\mathcal{L}_{\mathrm{per}}(\hat{x}_{\mathrm{sr}},\hat{x}_{\mathrm{hr}}),
\label{eq:stage2_img}
\end{equation}
and for videos 
\begin{align}
\mathcal{L}_{\mathrm{s2,vid}}
&= \|x_{\mathrm{sr}} - x_{\mathrm{hr}}\|_2^2
+ \lambda_{1}\,\mathcal{L}_{\mathrm{per}}(x_{\mathrm{sr}},x_{\mathrm{hr}})
+ \lambda_{2}\,\mathcal{L}_{\mathrm{frame}}(x_{\mathrm{sr}},x_{\mathrm{hr}}), \label{eq:stage2_vid} \\
\mathcal{L}_{\mathrm{frame}}
&= \frac{1}{n-1}\sum_{i=2}^{n}
\bigl\|\Delta x^{(i)}_{\mathrm{sr}} - \Delta x^{(i)}_{\mathrm{hr}}\bigr\|_{1},
\quad
\Delta x^{(i)} = x^{(i)} - x^{(i-1)}.
\label{eq:frame_diff}
\end{align}
This yields efficient adaptation (Stage 1) and high-fidelity restoration (Stage 2).

\subsection{Unified token processing for metadata}
We represent each metadata modality $m^{(k)}$ (e.g., edges, depth, segmentation) as an image tensor and encode it with the \emph{same} 3D VAE:
\begin{equation}
z_{\mathrm{meta}}^{(k)} = \mathcal{E}(m^{(k)}),
\qquad k=1,\ldots,K.
\label{eq:meta_vae}
\end{equation}
By construction, $z_{\mathrm{lr}}$ and $z_{\mathrm{meta}}^{(k)}$ are in the same latent space. We flatten all latents into tokens and concatenate:
\begin{equation}
\mathbf{Z} = [Z_{\mathrm{lr}}; Z_{\mathrm{T}}; Z_{\mathrm{meta}}], \quad
Z_{\mathrm{meta}} = [Z_{\mathrm{meta}}^{(1)};\dots; Z_{\mathrm{meta}}^{(K)}],
\end{equation}
where $Z_{\mathrm{lr}}$ are noisy visual tokens, $Z_{\mathrm{T}}$ are text tokens, and $Z_{\mathrm{meta}}$ collects metadata tokens.
MetaSR then applies the original DiT blocks:
\begin{equation}
\mathbf{H}^{(\ell+1)} = \mathrm{DiTBlock}^{(\ell)}(\mathbf{H}^{(\ell)};t).
\end{equation}
All tokens attend to each other via multi-modal attention, enabling both spatial and semantic conditioning without new modules.

\paragraph{Position-aware token interaction.}
CogVideoX-2B uses a fixed 3D sinusoidal positional encoding for visual tokens. We assign the same positional indices to aligned metadata tokens, ensuring spatial consistency.
In this way, metadata is fused through the backbone’s native attention, and in the next section we prove that such side information provably reduces uncertainty.

\subsection{Posterior Entropy Suppression via Verifiable Metadata}
\label{sec:theory_entropy_suppression}

\paragraph{Setup.}
Let $X$ denote the target (HR) image and $Y$ its degraded observation. Let $M$ be structured metadata (e.g.\ depth, edges) extracted at the encoder.
We use $q(x,y,m)$ for the joint data distribution.

\begin{theorem}[Entropy reduction under side information]
\label{thm:entropy_reduction}
For any joint $q(x,y,m)$,
\begin{align}
H_q(X\mid Y,M) &\le H_q(X\mid Y), \\
H_q(X\mid Y)-H_q(X\mid Y,M) &= I_q(X;M\mid Y) \ge 0.
\end{align}
\end{theorem}

\begin{proof}
By definition, $I_q(X;M\mid Y)=H_q(X\mid Y)-H_q(X\mid Y,M)$. Also,
\begin{equation}
I_q(X;M\mid Y)
=\mathbb{E}_{q(Y)}\bigl[D_{\mathrm{KL}}(q(X,M\mid Y)\,\Vert\,q(X\mid Y)q(M\mid Y))\bigr]\ge0.
\end{equation}
Rearranging proves the result.
\end{proof}

\paragraph{Verification-gated conditioning.}
To handle unreliable $M$, we introduce a gate $V=v(Y,M)\in\{0,1\}$.
Set $\tilde M=M$ if $V=1$, else $\tilde M=\bot$.
Then by Theorem~\ref{thm:entropy_reduction},
\[
H_q(X\mid Y,\tilde M) \le H_q(X\mid Y),
\]
ensuring gating never increases entropy.

\paragraph{Implication for log-likelihood training.}
For model $p_\theta(x\mid y,\tilde m)$, the expected NLL is
\[
\mathbb{E}_{q}[-\log p_\theta(X\mid Y,\tilde M)]
= H_q(X\mid Y,\tilde M) + \mathbb{E}_{q}[D_{\mathrm{KL}}(q(\cdot)\,\Vert\,p_\theta(\cdot))] \ge H_q(X\mid Y,\tilde M).
\]
Thus if $I_q(X;M\mid Y)>0$, conditioning on $M$ lowers the optimal log-loss.

\paragraph{Connection to metadata bitrate.}
Since $I_q(X;M\mid Y)\le H_q(M)$, the information gain is bounded by the metadata entropy (and thus bitrate).
This motivates allocating metadata bits to maximize $I_q(X;M\mid Y)$ under budget, as evaluated in Sec.~\ref{sec:experiments}.

\section{Experiments}
\label{sec:experiments}
\subsection{Training Data and Fine-Tuning Setup}
\label{sec:training_setup}
\paragraph{Datasets.}
We fine-tune CogVideoX-2B, a pretrained video diffusion backbone, on a mixed corpus of DIV2K images and HQ-VSR videos.
DIV2K contains 1,000 2K-resolution images with an 800/100 train/test split, while HQ-VSR comprises 2{,}055 high-quality videos for video super-resolution.
All implementations are based on PyTorch, and training is conducted on two NVIDIA A100 GPUs. We adopt a progressive-resolution curriculum from 320$\times$640 to 720$\times$1280 and run two-stage fine-tuning for fewer than 10{,}000 optimization steps.

Although the proposed framework is inherently applicable to video super-resolution, all experiments reported in this paper are performed on a frame-converted subset of UDM videos.
This design choice is driven by two practical considerations: (i) there is no off-the-shelf codec tailored for compressing 2-bit edge-like side signals in video form, and (ii) frame-level evaluation helps maintain fair comparison with existing baselines.


\subsection{Posterior Uncertainty in the Compression--Transmission--Generation Pipeline}
\label{sec:posterior_uncertainty}
This subsection shows that, in a unified compression, transmission, and generation system, the posterior over plausible HR reconstructions is inherently uncertain. We use controlled degradations and representative SR baselines to expose the sources of this uncertainty before introducing our metadata-aware RDO analysis.

\paragraph{JPEG-based bitrate sweep.}
To study the impact of compression artifacts under controlled bitrates, we apply JPEG compression to non-downsampled images while varying the quality factor to obtain different frame bitrates. This setup isolates compression artifacts from classical downsampling-induced information loss.

\paragraph{Synthetic transmission degradations.}
To emulate sender--receiver channel imperfections, we define two degradation levels on top of the JPEG-compressed data. Starting from the no-noise case, we add zero-mean gray-scale Gaussian noise (LN: $\sigma=10$, HN: $\sigma=20$), followed by isotropic Gaussian blur (LN: kernel size $5$, $\sigma=0.8$; HN: kernel size $7$, $\sigma=1.2$). The resulting distortion trend is summarized by PSNR/SSIM in Fig.~\ref{fig:poison}(a).

\paragraph{Baselines and metrics.}
For Fig.~\ref{fig:poison}(b), we compare three representative SR methods: (i) EDSR (CNN regression SR), (ii) StableSR (diffusion-prior blind image SR), and (iii) DOVE (CogVideoX-adapted diffusion model for real-world video SR). We report PSNR and SSIM as full-reference fidelity/structure metrics.

\paragraph{Implication for posterior uncertainty.}
Fig.~\ref{fig:poison}(a) shows that compression and channel degradation jointly reduce reconstruction quality. Fig.~\ref{fig:poison}(b) further shows that even without explicit spatial downsampling, generative SR models may introduce stronger hallucination-like deviations due to prior-driven reconstruction. Together, these observations indicate non-negligible posterior uncertainty in practical systems. This directly motivates the next subsection, where we use a metadata-aware RDO framework to quantify how much side information is needed to reduce this uncertainty and improve reconstruction quality.

\begin{figure}[t]
  \centering
  \includegraphics[width=\linewidth]{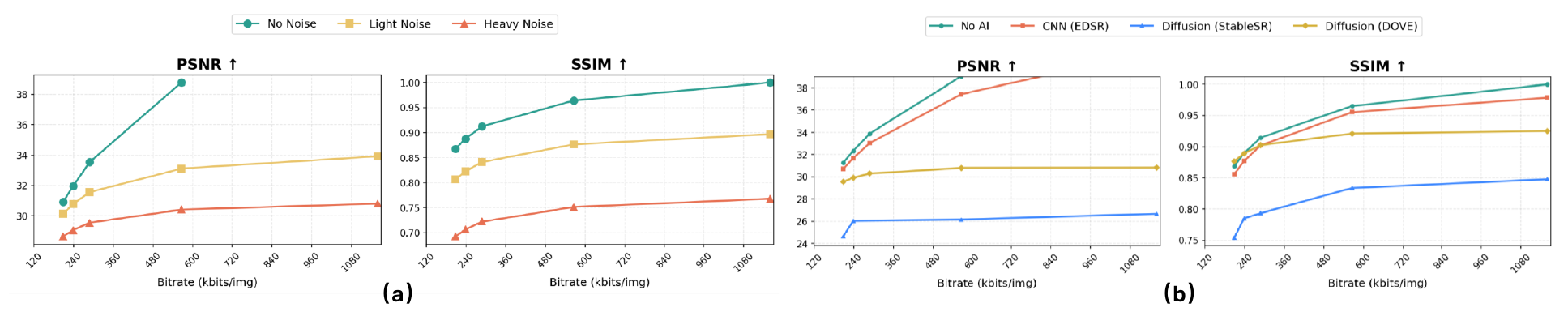}
  \caption{Case study of posterior uncertainty in a compression, transmission, and generation pipeline. Panel (a) reports quality degradation under bitrate and channel corruption, and panel (b) compares EDSR, StableSR, and DOVE to highlight prior-driven hallucination effects.}
  \label{fig:poison}
\end{figure}

\subsection{Metadata Rate--Distortion Evaluation}
\label{sec:fig4}

We evaluate MetaSR in a sender--receiver setting where the sender transmits (i) a JPEG-compressed base layer and (ii) a compact metadata stream, and the receiver reconstructs high-quality outputs using either DOVE or MetaSR. Our goal is to characterize the metadata rate--distortion (R--D) trade-off under different transmission conditions, and to use an RDO framework to quantify how metadata bitrate reduces posterior uncertainty and improves reconstruction quality. Unless otherwise specified, all experiments in this section are conducted under 4$\times$ super-resolution. In LN/HN settings, the LQ input to both DOVE and MetaSR is degraded using the same corruption pipeline as in Sec.~\ref{sec:posterior_uncertainty}.

\paragraph{RDO criterion for end-to-end SR.}
We adopt an operational RDO criterion~\cite{schuster2013rate} that jointly evaluates transmission cost and reconstruction quality:
\begin{equation}
J = D + \lambda R,
\end{equation}
where $R$ is the total sender bitrate (JPEG base layer + metadata) and $D$ is the receiver/display distortion term (e.g., based on PSNR or SSIM). This criterion is important because quality gains are only practically useful when achieved under reasonable transmission cost.

\paragraph{Metadata, JBIG2 compression, and bitrate accounting.}
We use Canny edge maps~\cite{canny1986} as structured metadata and compute edges on the HR reference at the sender. We compress the binary edge maps using JBIG2, which is a standardized codec for bi-level images and supports lossless coding. For each noise regime, we sweep the metadata bitrate by varying the edge-map sparsity/quantization and compressing the resulting representation with lossless JBIG2. At each rate point, we reconstruct using DOVE and MetaSR and report PSNR, SSIM, LPIPS, DISTS, and CLIP-IQA~\cite{wang2004ssim,zhang2018lpips,ding2020dists,wang2022clipiqa}. 

As shown in Fig.~\ref{fig:rdo_curves}, metadata guidance consistently improves the R--D behavior of MetaSR over DOVE at the same total transmission budget (base layer + metadata) across all transmission regimes. The gain becomes larger as channel degradation increases (HN $>$ LN $>$ NN), indicating that structured side information is most beneficial when the LQ observation is heavily corrupted by noise and blur. The improvement is most visible on structure-aware and perceptual metrics, including SSIM, LPIPS, DISTS, and CLIP-IQA. In the HN regime, MetaSR reaches up to 1.0~dB PSNR gain over DOVE at matched bitrate and achieves up to 50\% bitrate saving at matched PSNR. The NN regime is used as an idealized upper-bound reference, since real-world image/video delivery is rarely truly noise-free.

\paragraph{Qualitative perception under HN degradation.}
Beyond metric curves, MetaSR reconstructions in HN conditions show sharper contours, more coherent local textures, and fewer hallucinated artifacts than DOVE, especially around high-frequency edges and text-like structures. These visual trends are consistent with the gains observed in LPIPS, DISTS, and CLIP-IQA.

\paragraph{Generality and metadata selection under RDO.}
The metadata-guided paradigm is model-agnostic and can be extended to other AI SR backbones beyond CogVideoX/DiT. In general, informative metadata tends to improve reconstruction quality compared with no-metadata baselines. However, once the RDO objective is enforced, selecting which metadata components to transmit becomes a central design problem, because different modalities provide different quality gains per transmitted bit.

\begin{figure}[t]
  \centering
  \includegraphics[width=\linewidth]{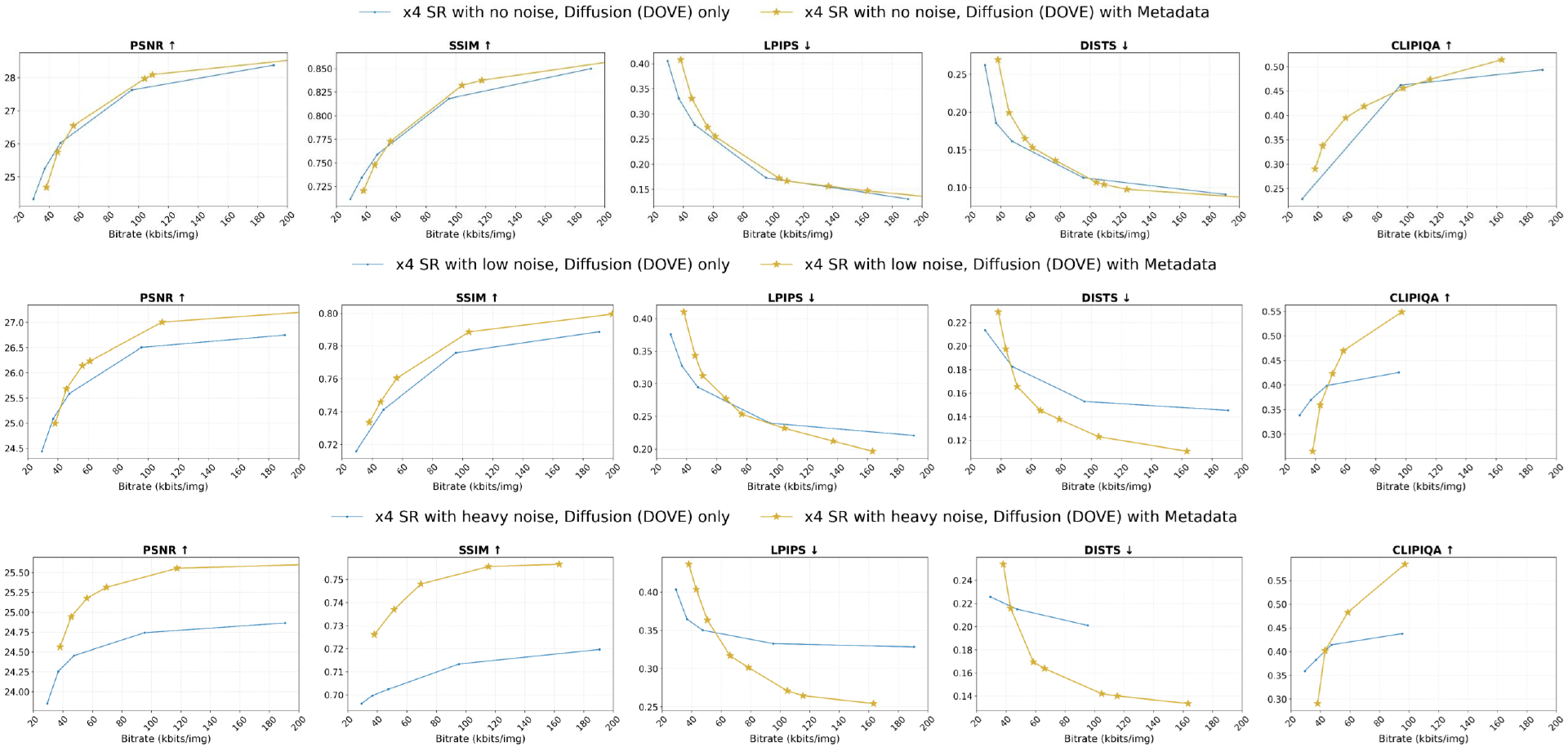}
  \caption{Rate--distortion (RDO) curves under three transmission-degradation regimes (NN/LN/HN), comparing DOVE and MetaSR at identical total transmission budgets (JPEG base layer + metadata). The performance gap widens from NN to HN, showing stronger metadata benefits under harsher channel corruption.}
  \label{fig:rdo_curves}
\end{figure}

\paragraph{Preliminary feasibility beyond SR.}
Fig.~\ref{fig:future_work} presents a preliminary example of Canny-edge-guided video frame interpolation. This result demonstrates the feasibility of extending the proposed metadata-conditioned pipeline beyond SR, and our initial experiments are encouraging for future development.

\begin{figure}[t]
  \centering
  \includegraphics[width=\linewidth]{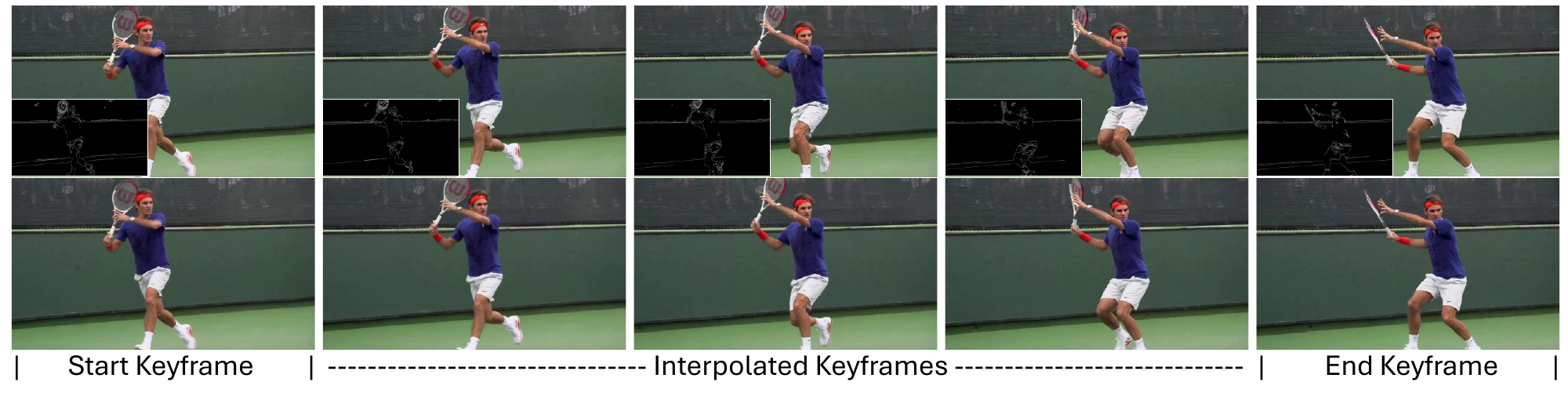}
  \caption{Preliminary future-work result: Canny-edge-guided video frame interpolation. The example shows the feasibility of extending MetaSR-style metadata conditioning beyond super-resolution.}
  \label{fig:future_work}
\end{figure}

\section{Conclusions}
\label{sec:conclusions}
We presented MetaSR, a sender--receiver collaborative framework for generative super-resolution that treats metadata as a first-class, bitrate-constrained side-information stream rather than a fixed auxiliary input. Built on CogVideoX-2B with one-step diffusion adaptation, MetaSR introduces a unified conditioning interface that reuses the native 3D VAE and DiT architecture, enabling flexible fusion of heterogeneous metadata without architectural redesign. We further provided an information-theoretic analysis showing that informative metadata reduces conditional uncertainty, which motivates metadata orchestration under transmission budgets.

Extensive experiments validate this design from both uncertainty and rate--distortion perspectives. Across transmission regimes, MetaSR consistently outperforms DOVE at matched total bitrate, and the gains become more pronounced under stronger degradations. Improvements are especially clear on structure-aware and perceptual metrics, and substantial bitrate savings are achieved at matched PSNR in degraded channels. These results support the central claim of this work: structured metadata can effectively suppress posterior ambiguity in practical compression--transmission--generation pipelines.

This study also has limitations. Current evaluation is conducted on a frame-converted subset due to the lack of off-the-shelf codecs for compact edge-like video side signals, and metadata modalities are limited to representative structured cues. Future work will focus on video-native metadata compression, temporally consistent metadata orchestration, adaptive reliability-aware gating under realistic channels, and broader metadata families for a wider range of tasks (e.g., video frame interpolation and related restoration problems), as preliminarily demonstrated in Fig.~\ref{fig:future_work}, to further improve robustness and deployment efficiency in real-world systems.

\section*{Acknowledgements}
The authors would like to thank Jianliang Yi, Weiqiang Lei and other colleagues from TCL for their valuable support throughout this work, particularly for their assistance with code implementation, technical advices and computational supports.

\bibliographystyle{splncs04}
\bibliography{main}

\end{document}